\documentclass[letterpaper, 10 pt, conference]{ieeeconf}
\IEEEoverridecommandlockouts

\usepackage[colorlinks]{hyperref}
\hypersetup{
    colorlinks,
    linkcolor=black,
    citecolor=black,
    filecolor=magenta,
    urlcolor=cyan,
}

\usepackage{resizegather}

\usepackage{cite}
\usepackage{amsmath,amssymb,amsfonts,mathrsfs}
\usepackage{algorithmic}
\usepackage{graphicx}
\usepackage{textcomp}
\usepackage{xcolor}
\usepackage{tcolorbox}

\usepackage{tablefootnote}
\usepackage{threeparttable}
\usepackage{caption, subcaption}

\usepackage{tikz}
\usepackage{graphicx}
\usepackage{tkz-euclide}

\usetikzlibrary{positioning}
\usetikzlibrary{shapes}
\usetikzlibrary{shapes.misc}
\usetikzlibrary{shapes.geometric}
\usetikzlibrary{plotmarks}
\usetikzlibrary{intersections}
\usetikzlibrary{calc}
\usetikzlibrary{fit}
\usetikzlibrary{patterns,tikzmark}
\usetikzlibrary{matrix,decorations.pathreplacing,calc}

\tikzset{cross/.style={cross out, draw, 
         minimum size=2*(#1-\pgflinewidth), 
         inner sep=0pt, outer sep=0pt}}

\DeclareMathOperator*{\argmin}{arg\,min}

\newcommand{\state}[0]{x}
\newcommand{\prel}[0]{p_{\rm{rel}}}
\newcommand{\vrel}[0]{v_{\rm{rel}}}
\newcommand{\preldot}[0]{\dot{p}_{\rm{rel}}}

\newcommand{\norm}[1]{\left\lVert#1\right\rVert}

\newcommand{\vertiii}[1]{{\left\vert\kern-0.25ex\left\vert\kern-0.25ex\left\vert #1 
    \right\vert\kern-0.25ex\right\vert\kern-0.25ex\right\vert}}

\newtheorem{theorem}{Theorem}

\newtheorem{definition}{Definition}

\makeatletter
\renewcommand{\fps@figure}{htp}
\renewcommand{\fps@table}{htp}
\makeatother

\def\BibTeX{{\rm B\kern-.05em{\sc i\kern-.025em b}\kern-.08em
    T\kern-.1667em\lower.7ex\hbox{E}\kern-.125emX}}
    
\begin{document}

\title{Control Barrier Functions in Dynamic UAVs for Kinematic Obstacle Avoidance: A Collision Cone Approach}

\author{Manan Tayal$^{1}$, Rajpal Singh$^{2}$, Jishnu Keshavan$^{2}$, Shishir Kolathaya$^{1}$
\thanks{This research is supported by ARTPARK.
}
\thanks{$^{1}$Cyber-Physical Systems, Indian Institute of Science (IISc), Bengaluru.
{\tt\scriptsize \{manantayal, shishirk\}@iisc.ac.in}
.
}%
\thanks{$^{2}$Department of Mechanical Engineering, Indian Institute of Science (IISc), Bengaluru.
{\tt\scriptsize \{rajpalsingh, kjishnu\}@iisc.ac.in}
.
}%
}

\maketitle
\begin{abstract}
Unmanned aerial vehicles (UAVs), specifically quadrotors, have revolutionized various industries with their maneuverability and versatility, but their safe operation in dynamic environments heavily relies on effective collision avoidance techniques. This paper introduces a novel technique for safely navigating a quadrotor along a desired route while avoiding kinematic obstacles. We propose a new constraint formulation that employs control barrier functions (CBFs) and collision cones to ensure that the relative velocity between the quadrotor and the obstacle always avoids a cone of vectors that may lead to a collision. By showing that the proposed constraint is a valid CBF for quadrotors, we are able to leverage its real-time implementation via Quadratic Programs (QPs), called the CBF-QPs. Validation includes PyBullet simulations and hardware experiments on Crazyflie 2.1, demonstrating effectiveness in static and moving obstacle scenarios. Comparative analysis with literature, especially higher order CBF-QPs, highlights the proposed approach's less conservative nature.
\end{abstract}


\section{Introduction}
\label{section: Introduction}
\par Quadrotors are used in a wide range of applications, including search and rescue, environmental monitoring, agriculture, transportation, and entertainment \cite{kumar2015future}. In many of these applications, quadrotors operate in complex and dynamic environments, where they must navigate around obstacles such as trees, buildings, and other drones. The literature presents a variety of methods such as artificial potential field \cite{8022685}, reachability analysis \cite{8263977} \cite{RA-UAV}, and nonlinear model predictive control \cite{8442967} to address the problem of obstacle avoidance in UAVs.

In recent years, the Control Barrier Functions (CBFs) based approach \cite{7040372}\cite{Ames_2017} has emerged as a promising strategy for ensuring safe operation of autonomous systems. This is a model-based control design method, which provides a computationally efficient solution that can handle complex situations while guaranteeing safety. CBFs can be formulated as a Quadratic Problem (QP) and can be solved online, making them well-suited for real-time safety-critical applications. 
CBFs are specifically designed to enforce safety constraints and provide hard constraints on the system's trajectory, making them superior to Nonlinear MPC (NMPC) in terms of safety guarantees. NMPC, on the other hand, provides soft constraints on the system's trajectory, with the degree of constraint satisfaction dependent on the optimization algorithm's performance.

\begin{figure}[t]
    \centering
    \includegraphics[width=0.50\linewidth]{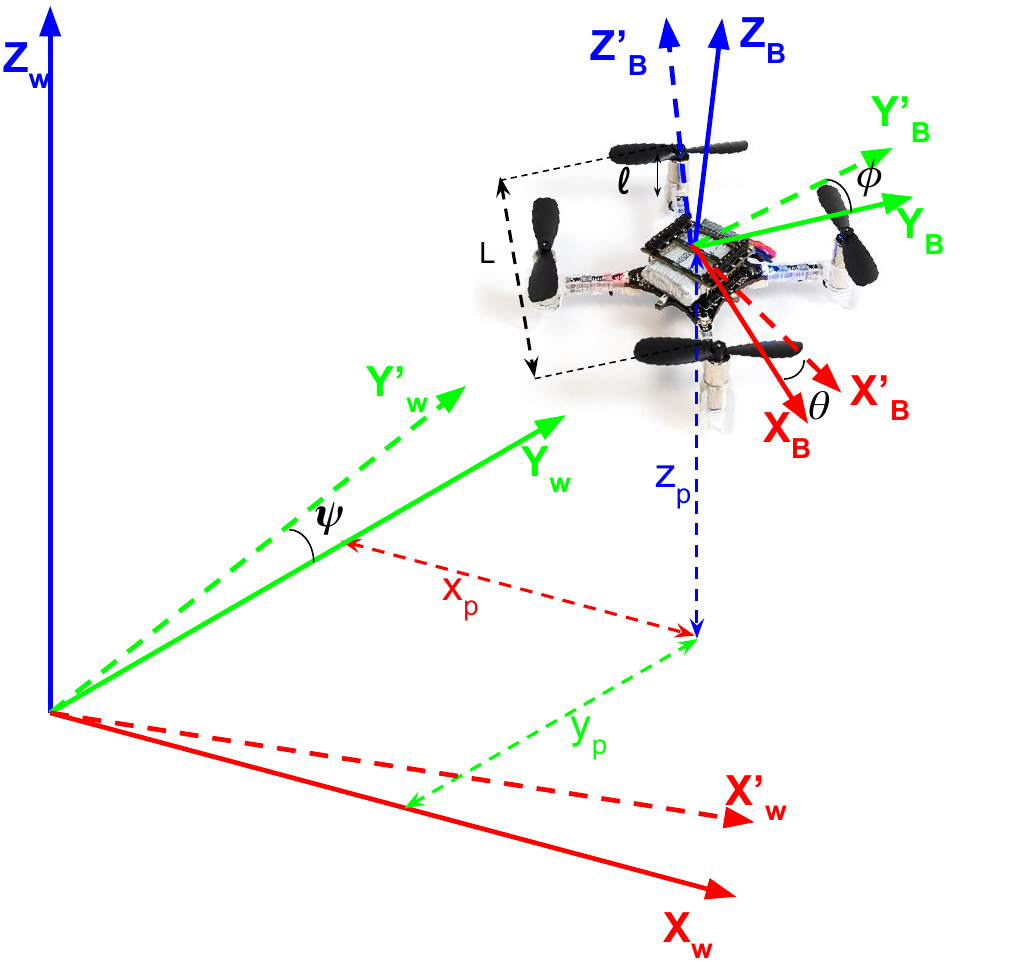}
\caption{World coordinates and body fixed coordinates of Crazyflie and Euler's angles defined in these coordinates}
\label{fig:models}
\end{figure}

In situations that involve complex interactions between subsystems or where safety requirements are highly dynamic and subject to frequent changes, reachability analysis may be limited. In such cases, CBFs are more suitable due to their ability to handle highly dynamic safety requirements \cite{https://doi.org/10.48550/arxiv.2106.13176}. While the artificial potential field approach is easy to implement, it suffers from limitations such as the possibility of getting stuck in local minima and difficulties in handling complex environments with multiple obstacles and \cite{Singletary2021ComparativeAO} shows that CBFs offer a viable, and arguably improved alternative to APFs for real-time obstacle avoidance.  Thus, the Control Barrier Functions approach provides a better solution for obstacle avoidance in UAVs, particularly in safety-critical scenarios. \cite{7525253} has shown collision avoidance using CBFs in planar quadrotor case.
In a recent work, \cite{DBLP:journals/corr/abs-1903-04706, 9516971} showed that Higher order CBFs are a generalized form of the exponential CBFs, thus it also addresses this problem of obstacle avoidance. However, a major challenge with this approach is the need to identify suitable penalty parameters (p's) and class $\mathcal{K}$ functions ($\alpha$'s) that can yield optimal results.

With regards to obstacle avoidance in dynamic environments, another class of approaches that is widely used is the method of collision cones \cite{Fiorini1993}, \cite{ doi:10.1177/027836499801700706}, \cite{709600}
It involves defining a cone-shaped region between two objects to represent the potential area of collision, which can be avoided by adjusting the object's trajectory to prevent the relative velocity vector from falling within the cone. This approach has several advantages, including its simplicity, efficiency, and adaptability to different environments. The method can be easily integrated into existing motion planning algorithms, takes into account the speed and direction of moving objects and the shape and size of potential obstacles, and can work in dynamic and unpredictable environments. As a result, the Collision Cone approach provides a reliable and flexible means of avoiding collisions in various robotic and autonomous systems. 

Despite its simplicity and effectiveness, the collision cone method has largely been restricted to offline motion planning/navigation problems, and its extensions for real-time implementations have been limited. However, by exploiting the CBF-QP formulations, we can synthesize a new class of CBFs through the notion of collision cones, which can then be implemented in real-time. This will be the main objective of this paper. This idea was originally proposed in \cite{C3BF} for the planar case (2D) and for wheeled robots, while we aim to extend this for 3D and for quadrotors, which are underactuated and have higher degrees of freedom (DoF).

\subsection{Contribution and Paper Structure}
The main idea is to realize a CBF-QP formulation for the quadrotor dynamics and for obstacles with non-zero velocity values.
The main contributions of our work are:

\begin{itemize}
    \item We formulate a direct method for safe trajectory tracking of quadrotors based on collision cone control barrier functions expressed through a quadratic program.
    \item We consider static and constant velocity obstacles of various dimensions and provide mathematical guarantees for collision avoidance.
    \item We compare the collision cone CBF with the state-of-the-art higher-order CBF (HO-CBF), and show how the former is better in terms of feasibility and safety guarantees. 
\end{itemize}

\subsection{Organisation}
The rest of this paper is organized as follows. Preliminaries explaining the quadrotor model, the concept of control barrier functions (CBFs), collision cone CBFs, and controller design are introduced in section \ref{section: Background}. The application of the above CBFs on the quadrotor to avoid obstacles of various shapes is discussed in section \ref{section: Safety Guarantee}. The simulation setup and results will be discussed in section \ref{section: Simulation Results}. Finally, we present our conclusion in section \ref{section: Conclusions}.

\section{Preliminaries}
\label{section: Background}
In this section, first, we will describe the dynamics of quadrotor. Next, we will formally introduce Control Barrier Functions (CBFs) and their importance for real-time safety-critical control. Finally, we will introduce  Collision Cone Control Barrier Function (C3BF) approach.

\subsection{Quadrotor model}
The quadrotor model has four propellers, which provides upward thrusts of $(f_1, f_2, f_3, f_4)$ (see Fig. \ref{fig:models}) and the states needed to describe the quadrotor system is given by $x = [x_p, y_p, z_p, \dot{x}_p, \dot{y}_p, \dot{z}_p,  \phi, \theta, \psi, {\omega}_{1}, {\omega}_{2}, {\omega}_{3}]$ . The quadrotor dynamics is as follows~\cite{quadfolk}:
\begin{equation}
\label{eqn:quadrotor_model}
	\underbrace{\begin{bmatrix}
		\dot{x}_p \\
		\dot{y}_p \\
            \dot{z}_p \\
            \ddot{x}_p \\
		\ddot{y}_p \\
            \ddot{z}_p \\
		\dot{\phi} \\
            \dot{\theta} \\
            \dot{\psi} \\
		\dot{\omega}_{1} \\
            \dot{\omega}_{2} \\
            \dot{\omega}_{3}
	\end{bmatrix}}_{\dot{\state}}
	=
	\underbrace{\begin{bmatrix}
		\dot{x}_p \\
		\dot{y}_p \\
            \dot{z}_p \\
            0 \\
            0 \\
            -g \\
            \textbf{W}^{-1}
            \begin{bmatrix}
                \omega_{1} \\
                \omega_{2} \\
                \omega_{3}
            \end{bmatrix}\\
            \\
		  -I^{-1} \vec{\omega} \times I \vec{\omega}  \\
            .
	\end{bmatrix}}_{f(\state)}
	+
	\underbrace{\begin{bmatrix}
		\\
            \begin{bmatrix}
                0 & 0 & 0 & 0 \\
                0 & 0 & 0 & 0 \\
                0 & 0 & 0 & 0 
            \end{bmatrix}\\
            \\
            \frac{1}{m_Q}\textbf{R}
            \begin{bmatrix}
                0 & 0 & 0 & 0 \\
                0 & 0 & 0 & 0 \\
                1 & 1 & 1 & 1 
            \end{bmatrix}\\
            \\
            \begin{bmatrix}
                0 & 0 & 0 & 0 \\
                0 & 0 & 0 & 0 \\
                0 & 0 & 0 & 0 
            \end{bmatrix}\\
            \\
		I^{-1}L
            \begin{bmatrix}
                1 & 0 & -1 & 0 \\
                0 & 1 & 0 & -1 \\
                c_{\tau} & -c_{\tau} & c_{\tau} & -c_{\tau} 
            \end{bmatrix}
	\end{bmatrix}}_{g(\state)}
	\underbrace{\begin{bmatrix}
		f_{1} \\
		f_{2} \\
            f_{3} \\
            f_{4} 
	\end{bmatrix}}_{u}
\end{equation}
$x_p$, $y_p$, and $z_p$ denote the coordinates of the vehicle’s center of the base of the quadrotor in an inertial frame. $\phi$, $\theta$ and $\psi$ represents the (roll, pitch \& yaw) orientation of the quadrotor.  (see Fig. \ref{fig:models}). \textbf{R} is the rotation matrix (from the body frame to the inertial frame), $m_Q$ is the mass of the quadrotor, $\textbf{W}$ is the transformation matrix for angular velocities from the inertial frame to the body frame, $I$ is the inertia matrix and $L$ is the diagonal length of the quadrotor. $c_{\tau}$ is the constant
that determines the torque produced by each propeller. Note that even though we restrict our study to quadrotors in this paper, the extension of the proposed CBF-QPs for different multi-rotor UAVs is straightforward.

\subsection{Control barrier functions (CBFs)}
Having described the vehicle models, we now formally introduce Control Barrier Functions (CBFs) and their applications in the context of safety. 
%
Given the quadrotor model, we have the nonlinear control system in affine form:
\begin{equation}
	\dot{\state} = f(\state) + g(\state)u
	\label{eqn: affine control system}
\end{equation}
where $\state \in \mathcal{D} \subseteq \mathbb{R}^n$ is the state of system, and $u \in \mathbb{U} \subseteq \mathbb{R}^m$ the input for the system. Assume that the functions $f: \mathbb{R}^n \rightarrow \mathbb{R}^n$ and $g: \mathbb{R}^n \rightarrow \mathbb{R}^{n \times m}$ are continuously differentiable. Specific formulation of $f,g$ for the quadrotor were described in \eqref{eqn:quadrotor_model}. Given a Lipschitz continuous control law $u = k(\state)$, the resulting closed loop system $\dot{\state} = f_{cl}(\state) = f(\state) + g(\state)k(\state)$ yields a solution $\state(t)$, with initial condition $\state(0) = \state_0$.
%
Consider a set $\mathcal{C}$ defined as the \textit{super-level set} of a continuously differentiable function $h:\mathcal{D}\subseteq \mathbb{R}^n \rightarrow \mathbb{R}$ yielding,
\begin{align}
\label{eq:setc1}
	\mathcal{C}                        & = \{ \state \in \mathcal{D} \subset \mathbb{R}^n : h(\state) \geq 0\} \\
\label{eq:setc2}
	\partial\mathcal{C}                & = \{ \state \in \mathcal{D} \subset \mathbb{R}^n : h(\state) = 0\}\\
\label{eq:setc3}
	\text{Int}\left(\mathcal{C}\right) & = \{ \state \in \mathcal{D} \subset \mathbb{R}^n : h(\state) > 0\}
\end{align}
It is assumed that $\text{Int}\left(\mathcal{C}\right)$ is non-empty and $\mathcal{C}$ has no isolated points, i.e. $\text{Int}\left(\mathcal{C}\right) \neq \phi$ and $\overline{\text{Int}\left(\mathcal{C}\right)} = \mathcal{C}$. 
The system is safe w.r.t. the control law $u = k(\state)$ if
	$\forall \: \state(0) \in \mathcal{C} \implies \state(t) \in \mathcal{C} \;\;\; \forall t \geq 0$.
We can mathematically verify if the controller $k(\state)$ is safeguarding or not by using Control Barrier Functions (CBFs), which is defined next.

\begin{definition}[Control barrier function (CBF)]{\it
\label{definition: CBF definition}
Given the set $\mathcal{C}$ defined by \eqref{eq:setc1}-\eqref{eq:setc3}, with $\frac{\partial h}{\partial \state}(\state) \neq 0\; \forall \state \in \partial \mathcal{C}$, the function $h$ is called the control barrier function (CBF) defined on the set $\mathcal{D}$, if there exists an extended \textit{class} $\mathcal{K}$ function $\kappa$ such that for all $\state \in \mathcal{D}$:

\begin{equation}
\begin{aligned}
    \underbrace{\text{sup}}_{ u \in \mathbb{U}}\! \left[\underbrace{\mathcal{L}_{f} h(\state) + \mathcal{L}_g h(\state)u} \iffalse+ \frac{\partial h}{\partial t}\fi_{\dot{h}\left(\state, u\right)} \! + \kappa\left(h(\state)\right)\right] \! \geq \! 0
\end{aligned}
\end{equation}
where $\mathcal{L}_{f} h(\state) = \frac{\partial h}{\partial \state}f(\state)$ and $\mathcal{L}_{g} h(\state)= \frac{\partial h}{\partial \state}g(\state)$ are the Lie derivatives. 
}
\end{definition}

Given this definition of a CBF, we know from \cite{Ames_2017} and \cite{8796030} that any Lipschitz continuous control law $k(\state)$ satisfying the inequality: $\dot{h} + \kappa( h )\geq 0$ ensures safety of $\mathcal{C}$ if $x(0)\in \mathcal{C}$, and asymptotic convergence to $\mathcal{C}$ if $x(0)$ is outside of $\mathcal{C}$. 

\subsection{Safety Filter Design}
\label{subsection: safe_controller}
Having described the CBF, we can now describe the Quadratic Programming (QP) formulation of CBFs. CBFs act as \textit{safety filters} which take the desired input $u_{des}(\state,t)$ and modify this input in a minimal way: 

\begin{equation}
\begin{aligned}
\label{eqn: CBF QP}
u^{*}(x,t) &= \argmin_{u \in \mathbb{U} \subseteq \mathbb{R}^m} \norm{u - u_{des}(x,t)}^2\\
\quad & \textrm{s.t. } \mathcal{L}_f h(x) + \mathcal{L}_g h(x)u + \kappa \left(h(x)\right) \geq 0\\
\end{aligned}
\end{equation}
This is called the Control Barrier Function based Quadratic Program (CBF-QP). The CBF-QP control $u^{*}$ can be obtained by solving the above optimization problem using KKT conditions.

\subsection{Collision Cone CBF (C3BF) candidate for quadrotors}
\label{subsection: C3BF}
We now formally introduce the proposed CBF candidate for quadrotors. Let us assume that the obstacle is centered at $(c_x(t), c_y(t), c_z(t))$ and with dimensions $c_1,c_2,c_3$. We assume that $c_x(t),c_y(t), c_z(t)$ are differentiable and their derivatives are piece-wise constants.
The proposed approach combines the idea of potential unsafe directions given by collision cone (Fig. \ref{fig:3D CBF}, \ref{fig:Projection CBF}) as an unsafe set to formulate a CBF as in \cite{C3BF}.
Consider the following CBF candidate:
\begin{equation}
    h(\state, t) = < \prel, \vrel> + \| \prel\|\| \vrel\|\cos\phi ,
    \label{eqn:CC-CBF}
\end{equation}
where $\prel$ is the relative position vector between the body center of the quadrotor and the center of the obstacle, $\vrel$ is the relative velocity, $<\cdot , \cdot>$ is the dot product of 2 vectors and $\phi$ is the half angle of the cone, the expression of $\cos\phi$ is given by $\frac{\sqrt{\|\prel\|^2 - r^2}}{\|\prel\|}$ (see Fig. \ref{fig:3D CBF}, \ref{fig:Projection CBF}). Precise mathematical definitions for $\prel, \vrel$ will be given in the next section. The proposed constraint simply ensures that the angle between $\prel, \vrel$ is less than $180^\circ - \phi$.

In \cite{C3BF}, it was shown that the proposed candidate \eqref{eqn:CC-CBF} is valid CBF for wheeled mobile robots, i.e., the unicycle and bicycle. With this result, CBF-QPs were constructed that yielded collision-avoiding behaviors in these models. We aim to extend this to the class of quadrotors. 

\section{Collision Cone CBFs on Quadrotor}
\label{section: Safety Guarantee}
Having described the Collision Cone CBF candidate, we will see their application on quadrotors in this section. Based on the shape of the obstacle we can divide the proposed candidates into two cases: 

%

\subsection{3D CBF candidate}
\label{section: 3D-CBF}
In scenarios where the dimensions of the obstacle are roughly equal, we can model the obstacle as a sphere, as illustrated in Fig. \ref{fig:3D CBF}. The resulting CBF in this context is referred to as the \textbf{3D CBF}. The relative position vector between the body center of the quadrotor and the center of the obstacle is as follows:
\begin{align}\label{eq:pos-vec-3D}
    \prel := \begin{bmatrix}
        c_x \\
        c_y \\
        c_z
    \end{bmatrix}
    - \left (
    \begin{bmatrix}
        x_p \\
        y_p \\
        z_p
    \end{bmatrix}
    + \textbf{R} \begin{bmatrix}
        0 \\
        0 \\
        l
    \end{bmatrix}
    \right )
\end{align}
Here $l$ is the distance of the body center from the base (see Fig. \ref{fig:models}). $c_x,c_y,c_z$ represents the obstacle location as a function of time. Also, since the obstacles are of constant velocity, we have $\Ddot{c}_x= \Ddot{c}_y= \Ddot{c}_z = 0$. We obtain its relative velocity as
\begin{align}\label{eq:vel-vec-3D}
    \vrel := \dot{p}_{rel}
\end{align}

\begin{figure}[t]
    \includegraphics[width=0.9\linewidth]{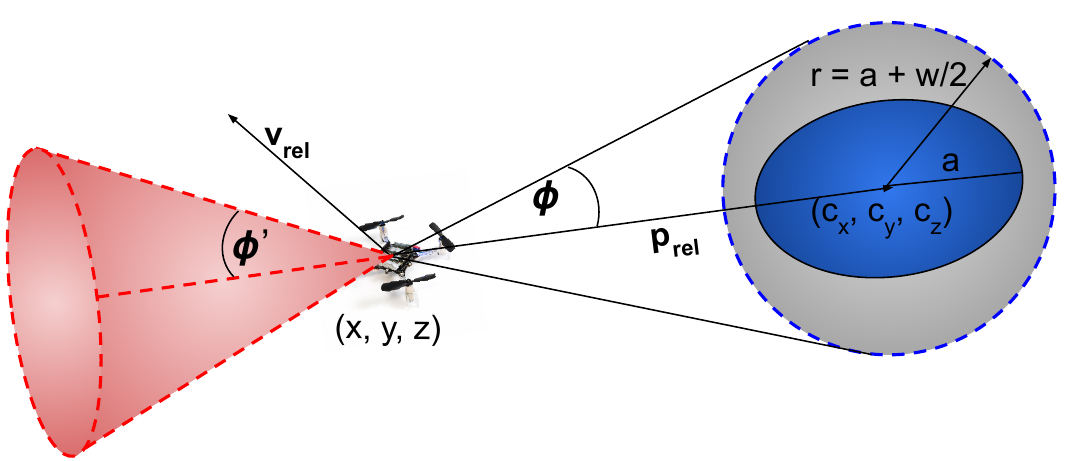}
\caption{\textbf{3D CBF} candidate: The dimensions of the obstacle are comparable to each other, it can be assumed as a sphere}
\label{fig:3D CBF}
\end{figure}

Having introduced Collision Cone CBF candidates in \ref{subsection: C3BF}, the next step is to formally verify that they are, indeed, valid CBFs. 
%
%
  %
We have the following result.

\begin{theorem}\label{thm:CC-CBF-3D}{\it
Given the quadrotor model \eqref{eqn:quadrotor_model}, the proposed CBF candidate \eqref{eqn:CC-CBF} with $\prel,\vrel$ defined by \eqref{eq:pos-vec-3D}, \eqref{eq:vel-vec-3D} is a valid CBF defined for the set $\mathcal{D}$.}
\end{theorem}
Please refer to \cite[Thm 3]{C3BF_tac} for the proof of Theorem.

\subsection{Projection CBF candidate}
\label{section: proj-CBF}

\begin{figure}[t]
    \centering
    \includegraphics[width=0.8\linewidth]{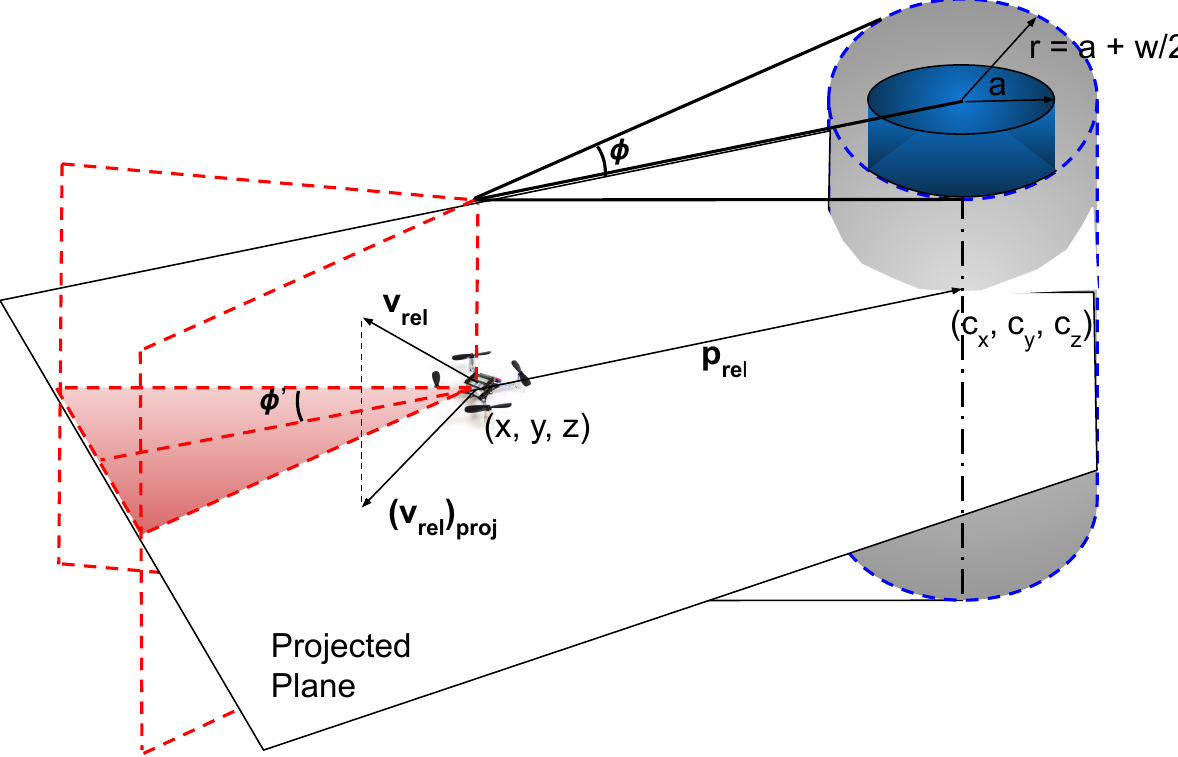}
\caption{\textbf{Projection CBF} candidate: One of the dimensions, of the obstacle, is bigger than the other dimensions, it can be assumed as a cylinder.}
\label{fig:Projection CBF}
\end{figure}
When an obstacle has significantly disparate dimensions, it can be approximated as a cylinder, giving rise to the \textbf{Projection CBF} (refer to Fig. \ref{fig:Projection CBF}). To derive this, we calculate the relative position vector between the quadrotor's body center and the intersection point of the obstacle's axis with the projection plane, which is perpendicular to the axis. Thus, we obtain:
\begin{align}\label{eqn:pos-vec-proj}
    (\prel)_{proj} := \mathcal{P}\left (\begin{bmatrix}
        c_x \\
        c_y \\
        c_z
    \end{bmatrix}
    - \left (
    \begin{bmatrix}
        x_p \\
        y_p \\
        z_p
    \end{bmatrix}
    + \textbf{R} \begin{bmatrix}
        0 \\
        0 \\
        l
    \end{bmatrix}
    \right ) \right ).
\end{align}
Here $l$ is the distance of the body center from the base (see Fig. \ref{fig:models}). $\mathcal{P}: \mathbb{R}^3 \to \mathbb{R}^3 $ is the projection operator, which can be assumed to be a constant\footnote{Note that the obstacles are always translating and not rotating. In addition, it is not restrictive to assume that the translation direction is always perpendicular to the cylinder axis. This makes the projection operator a constant.}. 
Now, since the relative position lies on the projection plane, we have one more condition to satisfy:
\begin{align}\label{eqn:p_in_plane}
    < (\prel)_{proj}, \hat{n} > = 0,   
\end{align}
where, $\hat{n}$ is the normal to the plane. Also, the relative velocity is given by:
\begin{align}\label{eqn:vel-vec-proj}
    (\vrel)_{proj} := \frac{d({p}_{rel})_{proj}}{dt} = (\preldot)_{proj}
\end{align}

\begin{theorem}\label{thm:CC-CBF-proj}{\it
Given the quadrotor model \eqref{eqn:quadrotor_model}, the proposed CBF candidate \eqref{eqn:CC-CBF} with $\prel,\vrel$ defined by \eqref{eqn:pos-vec-proj}, \eqref{eqn:vel-vec-proj} is a valid CBF defined for the set $\mathcal{D}$.}
\end{theorem}
Please refer to \cite[Thm 4]{C3BF_tac} for the proof of Theorem.

\subsection{Comparison with Higher Order CBFs}
We introduce the state-of-the-art Higher Order Control Barrier Functions (HO-CBFs) and compare them with the proposed C3BF in this section. Given that the collision constraints are with respect to position, the associated CBF has a relative degree of two. Therefore, it is necessary to establish a higher-order CBF with $m = 2$ as outlined in \cite[Eq. 16]{9516971}, which is expressed as: 
\begin{equation}
\begin{aligned}
\psi_{1}(x,t) &= \dot{b}(x,t) + p\alpha_{1} (b(x,t))\\
\psi_{2}(x,t) &= \dot{\psi_{1}}(x,t) + p \alpha_{2} (\psi_{1}(x,t)) ,\\
\end{aligned}
\end{equation}
where $b(\state,t) = (c_x(t) - x_p)^2 + (c_y(t) - y_p)^2 + (c_z(t) - z_p)^2 - r^2$, and r is the encompassing radius given by $r = max(c_1, c_2, c_3)$. $\alpha_1, \alpha_2$ are both class $\mathcal{K}$ functions, and $p$ is a tunable constant. As explained previously, $c_x,c_y,c_z$ is the center location of the obstacle as a function of time. Let us examine the form of HO-CBF where $\alpha_{1}$ is a square root function (which is also strictly increasing), and $\alpha_{2}$ is a linear function, due to its similarity to C3BF. Consequently, the resulting Higher Order CBF candidate takes the following form:
\begin{equation}
    h_{HO}(\state, t) = < \prel, \vrel> + \gamma \sqrt{(\|\prel\|^2 - r^2)}.
    \label{eqn:HO-CBF}
\end{equation}
%
We can show that the above-mentioned HO-CBF is also a valid CBF for quadrotors. We will now compare it with the proposed C3BF.

The C3BF concept aims to prevent the $\vrel$ vector, which represents the relative velocity between the quadrotor and the obstacle, from entering the collision cone region defined by the half-angle $\phi$. Figures \ref{fig:3D CBF}, \ref{fig:Projection CBF} and \ref{fig:HO-C3-CBF} illustrate this idea. We can rewrite the HO-CBF formula presented in \eqref{eqn:HO-CBF} in the following form: 
\begin{align}
h_{HO}(\state, t) = < \prel, \vrel> + \|\prel\| \|\vrel\| cos(\phi')
    \label{eqn:HO-CBF-CC}
\end{align}
where, $cos(\phi') = \frac{\gamma}{\|\vrel\|}cos(\phi)$.
If we are able to identify a suitable $\gamma$ (penalty term) for the given HO-CBF, it would result in a valid CBF as per \cite{9516971}. Nonetheless, in such a scenario where $\gamma$ remains constant and $\|\vrel\|$ goes on increasing, it leads to an increase in $\phi'$, thus, overestimating the cone as can be seen in Fig.\ref{fig:HO-C3-CBF}. Conversely, with the C3BF approach, we permit the penalty term to vary over time, i.e., $\gamma = \|\vrel\|$, resulting in a more precise estimation of the collision cone compared to the HO-CBF case. 
This also shows that C3BF is not a special case of Higher Order CBF.
This is also evident from the simulation outcomes of both CBFs, as demonstrated in Section \ref{section: Simulation Results}.

\begin{figure}[t]
    \includegraphics[width=0.9\linewidth]{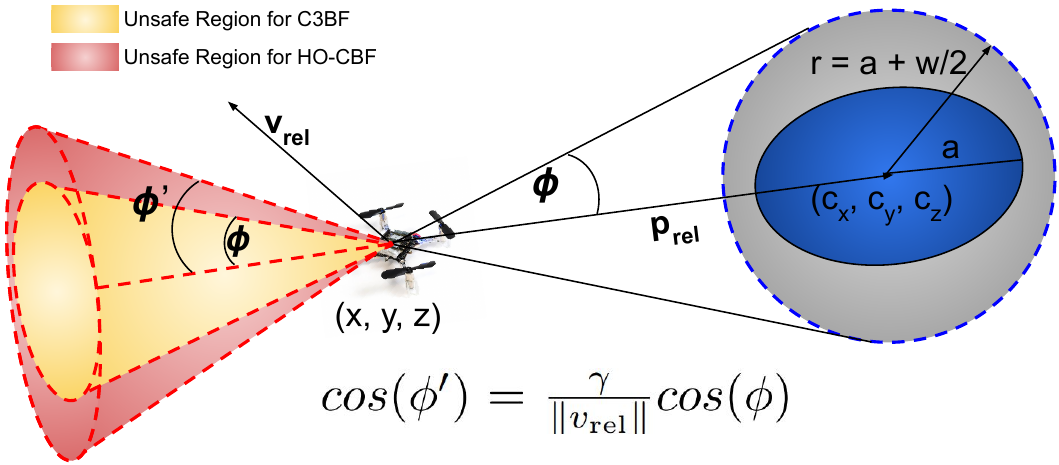}
\caption{Comparison of HO-CBF with C3BF. Here we are trying to compare the $\phi'$ and $\phi$ obtained from the two CBF formulations. It can be observed that $\phi'$ (pink cone) is dependent on $v_{rel}$, while $\phi$ (yellow cone) is a constant. The HO-CBF guarantees safety for a set that is not only smaller but also dependent on $v_{rel}$ as shown by the pink cone. Hence, HO-CBF is more conservative compared to C3BF.}
\label{fig:HO-C3-CBF}
\end{figure}

\section{Results and Discussions}
\label{section: Simulation Results}
\par We have validated the C3BF-QP based controller on quadrotors for both 3D and Projection CBF cases. We have used a PD controller as a reference controller to track the desired path with constant target velocities. Note that the reference controller can be replaced by any existing trajectory tracking/path-planning like the MPC \cite{9483029}. For the class $\mathcal{K}$ function in the CBF inequality, we chose $\kappa(h) = \gamma h$, where $\gamma=1$.

\begin{figure}
       \centering
        \begin{subfigure}[b]{0.46\textwidth}
        \includegraphics[width=\textwidth]{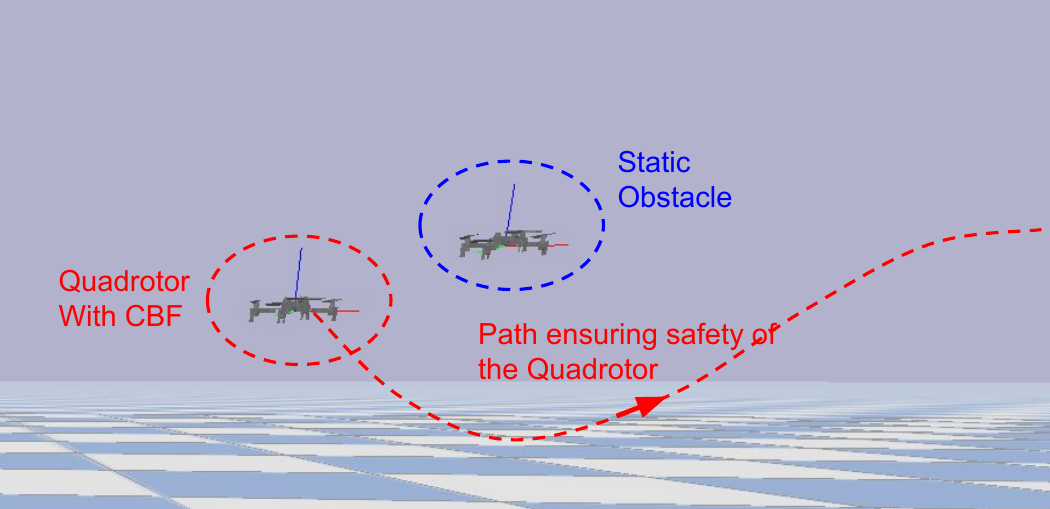}
        \caption{}
        \end{subfigure}
        \begin{subfigure}[b]{0.20\textwidth}
        \includegraphics[width=\textwidth]{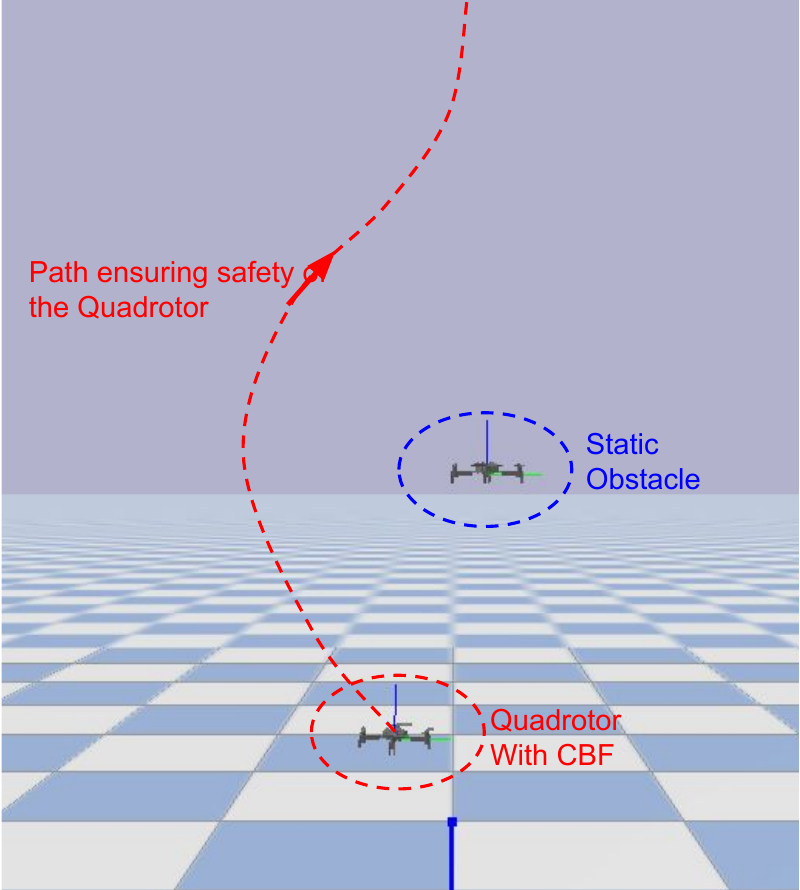}
        \caption{}
        \end{subfigure}
        \begin{subfigure}[b]{0.235\textwidth}
        \includegraphics[width=\textwidth]{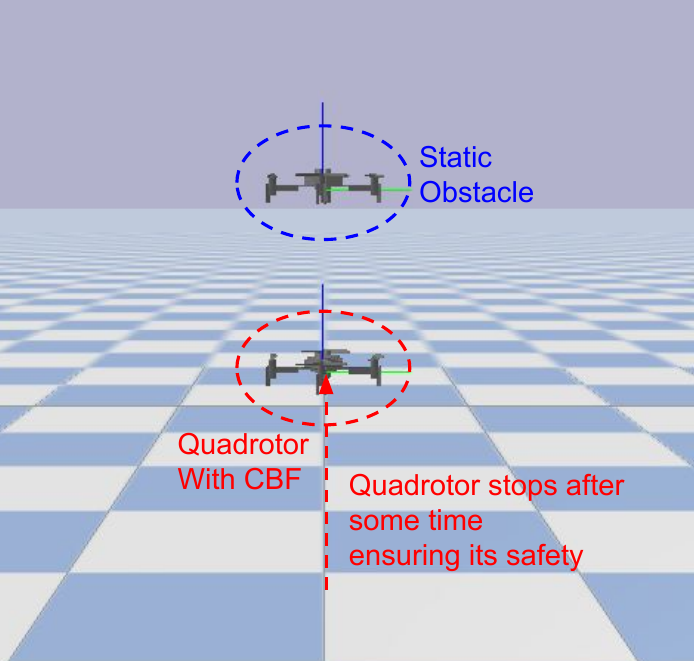}
        \caption{}
        \end{subfigure}
        \caption{Interaction with static obstacles: overtaking (a), (b), and braking (c) behavior of the quadrotor, Section \ref{section: 3D-CBF}.In all these cases the reference velocity of the quadrotor is 1m/s.}
        \label{fig:static-obs}
    \end{figure}

\begin{table}[b]

\begin{tabular}{|l|l|l|}
\hline
\textbf{Variables} & \textbf{Definition}         & \textbf{Value} \\ \hline
g                  & Gravitational acceleration  & $9.81 kg \cdot m/s^2$   \\ \hline
m                  & Mass of quadrotor           & $0.027 kg$        \\ \hline
L                  & Distance between two opp. rotors & 0.130 $m$         \\ \hline
l                  & Distance of center from base & $0.014 m$         \\ \hline
Ix, Iy             & Inertia about x, y-axis  & $2.39\cdot 10^{-5} kg \cdot m^2$    \\ \hline
Iz                 & Inertia about z-axis       & $3.23\cdot 10^{-5} kg \cdot m^2$    \\ \hline
kf                 & Motor’s thrust constant     & $3.16 \cdot 10^{-10}$          \\ \hline
km                 & Motor’s torque constant     & $7.94 \cdot 10^{-12}$          \\ \hline
\end{tabular}
\caption{Modelling parameters of Crazyflie}
\label{table:quadrotor_parameters}

\end{table}

\subsection{Simulation setup}
The simulations were conducted using the multi-drone environment \cite{pybullet-drones} on Pybullet \cite{coumans2019}, a Python-based physics simulation engine. The parameters of Crazyflie are tabulated in \ref{table:quadrotor_parameters}. Having presented our proposed control method design, we now test our framework under three different scenarios to illustrate the controller performance. These scenarios include the interaction of a quadrotor with (1) a static obstacle (3D case) Fig. \ref{fig:static-obs}, (2) a moving obstacle (3D case) Fig. \ref{fig:moving-obs} and (3) an elongated obstacle (Projection case) Fig. \ref{fig:long-obs}.


\begin{figure}
       \centering
        \begin{subfigure}[b]{0.33\textwidth}
        \includegraphics[width=\textwidth]{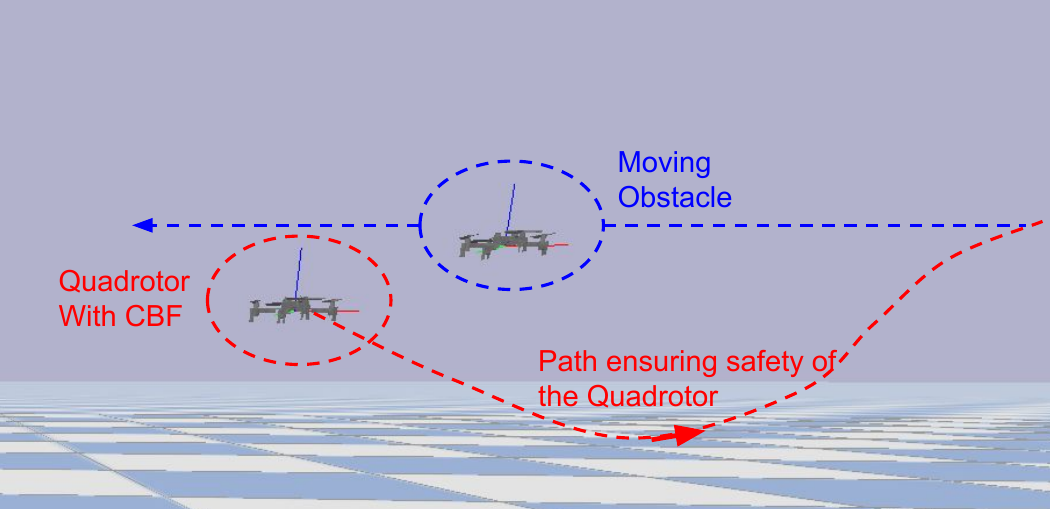}
        \caption{}
        \end{subfigure}
        \begin{subfigure}[b]{0.14\textwidth}
        \includegraphics[width=\textwidth]{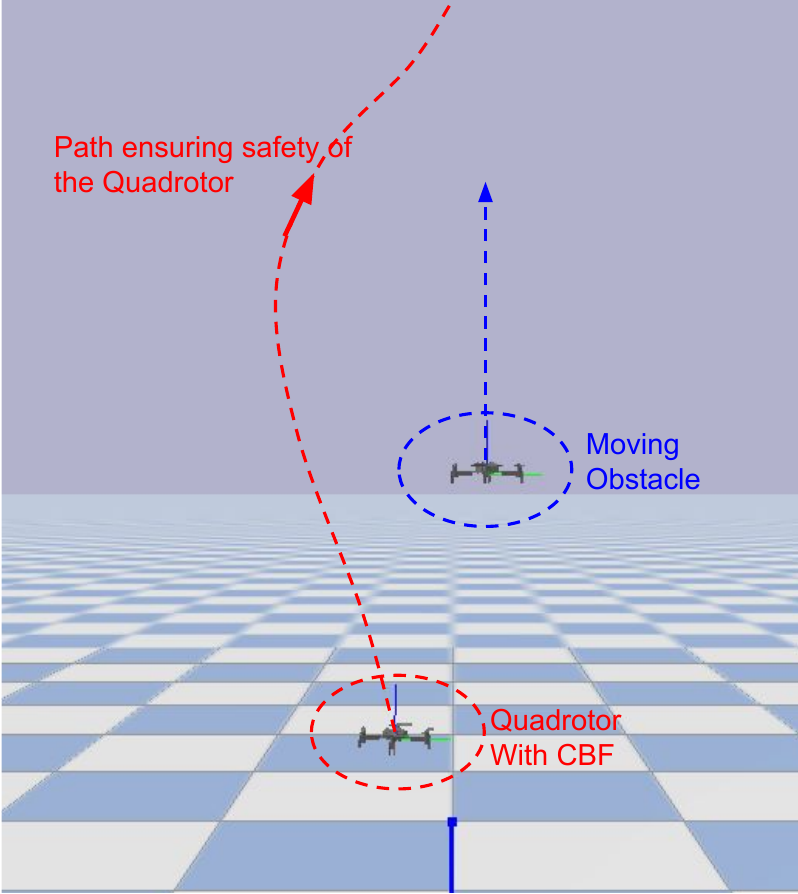}
        \caption{}
        \end{subfigure}
        \begin{subfigure}[b]{0.235\textwidth}
        \includegraphics[width=\textwidth]{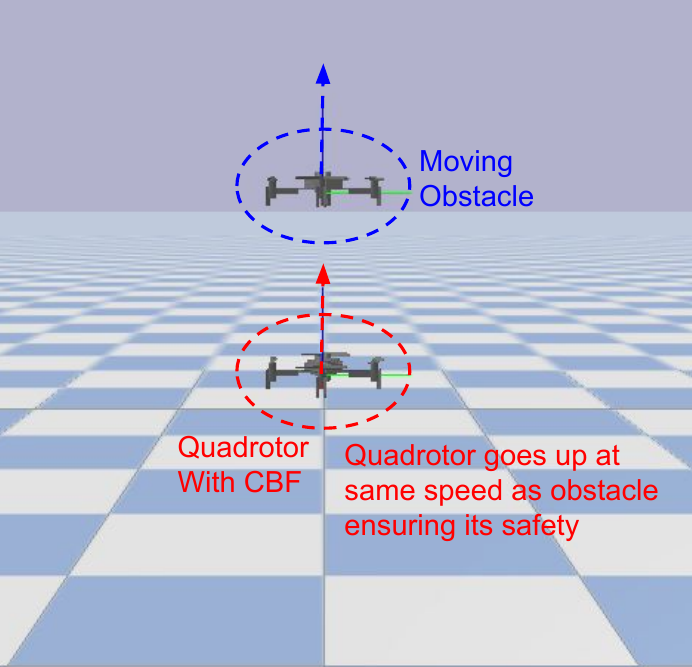}
        \caption{}
        \end{subfigure}
        \begin{subfigure}[b]{0.235\textwidth}
        \includegraphics[width=\textwidth]{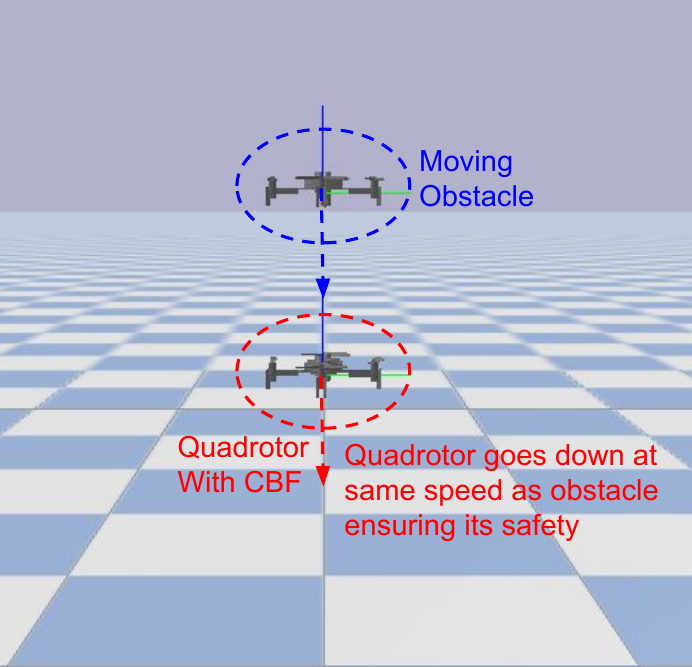}
        \caption{}
        \end{subfigure}
        \caption{Interaction with moving obstacles: overtaking (a), (b), slowing (c), and reversing (d) behavior of the quadrotor, section \ref{section: 3D-CBF}. In all these cases the reference velocity of the quadrotor is 1m/s and the obstacle quadrotor speed is 1m/s in case (a) and 0.1 m/s in (b),(c),(d).}
        \label{fig:moving-obs}
    \end{figure}


    \begin{figure}
       \centering
        \begin{subfigure}[b]{0.22\textwidth}
        \includegraphics[width=\textwidth]{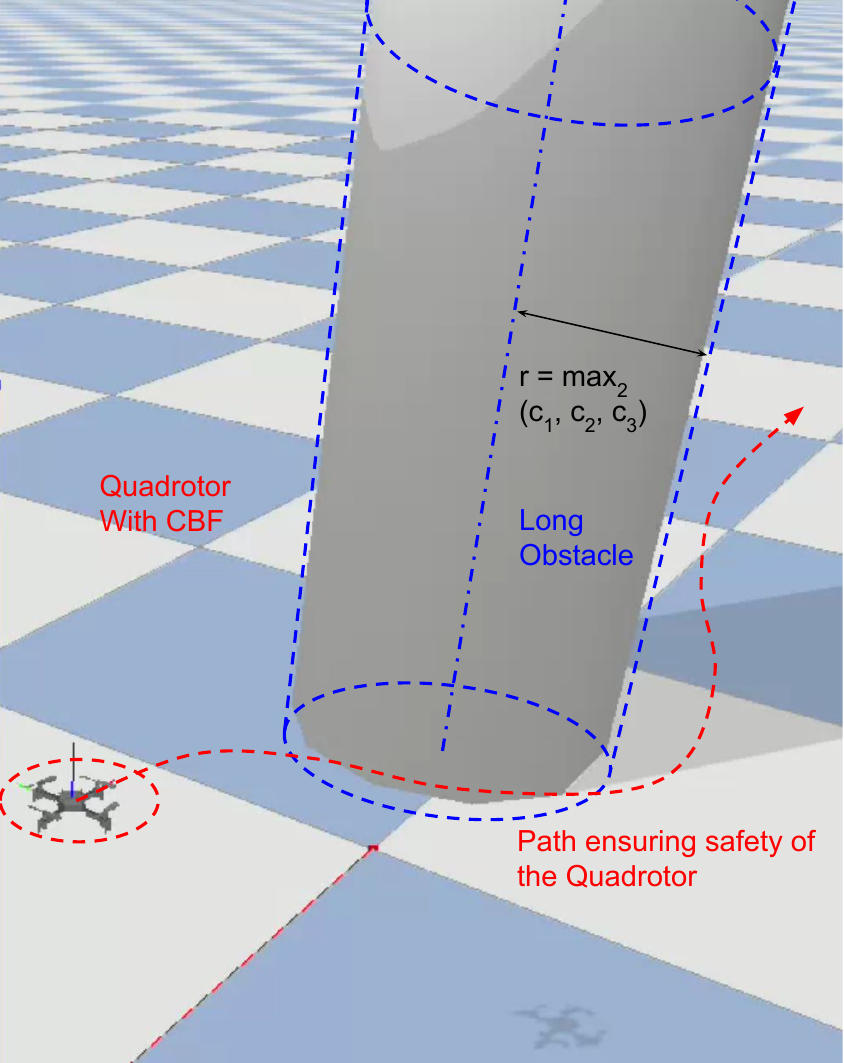}
        \caption{}
        \end{subfigure}
        \begin{subfigure}[b]{0.22\textwidth}
        \includegraphics[width=\textwidth]{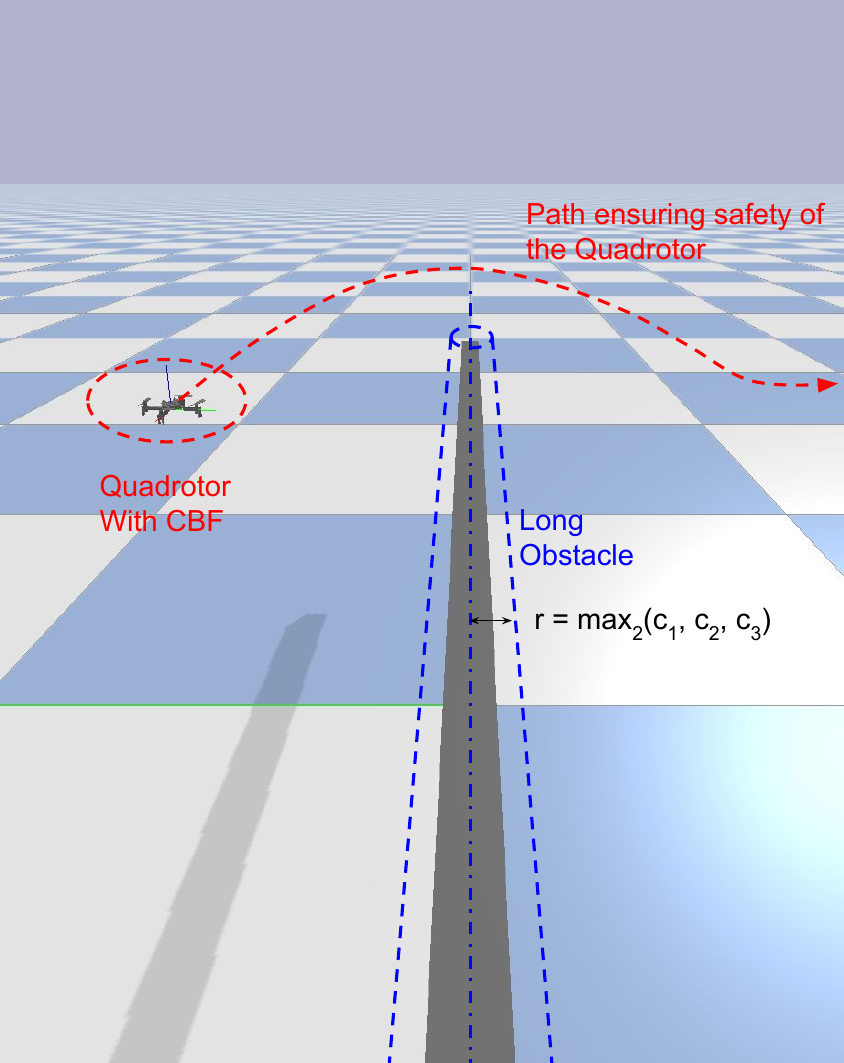}
        \caption{}
        \end{subfigure}
        \caption{Interaction with longer obstacles: moving from side (a) and top (b), Section \ref{section: proj-CBF}. In all these cases the reference velocity of the quadrotor is 1m/s. }
        \label{fig:long-obs}
    \end{figure}

\subsection{ Experimental Results}

The experimental results with Bitcraze\textsuperscript{\texttrademark} Crazyflie 2.1 aerial drone are presented to demonstrate the efficacy of the C3BF controller framework. The global position of the drone as well as the obstacle is measured using Qualisys\textsuperscript{\texttrademark} Miqus M3 motion capture system with a tracking frequency of 100 $Hz$. Further, for the drone, the global position from the motion capture system is fused with the onboard IMU data via the Extended Kalman Filter to get the filtered state. The control commands are generated by an off-board computer and transmitted to the drone via a radio link. The communication with the drone is facilitated through the Crazyflie Python library \cite{cfclient}. Experiments are performed for the cases with a single static obstacle, multiple static obstacles, and moving obstacles. The graphs and videos of hardware experiments are available here\footnote{\label{note: exp videos link} \url{https://tayalmanan28.github.io/C3BF-UAV/}}.
Hence, the experimental results verify the efficacy of the proposed scheme for obstacle avoidance.

\subsection{Comparison between C3BF and HO-CBF}
All the aforementioned cases were tested with the HO-CBF to compare its performance against C3BF. We observe that the HO-CBF could not avoid a high-speed approaching obstacle. Moreover, it is not able to properly avoid the longer obstacles in the projection CBF case. These shortcomings of the Higher Order CBF are also demonstrated in the simulation video available here$^{\ref{note: exp videos link}}$.

\subsection{Robustness of C3BF}
Without changing the above control framework we can observe that the C3BF is robust in the following two cases: 

\subsubsection{Multiple Obstacles}
The quadrotor successfully navigates through a series of obstacles (both Spherical and Long obstacles) avoiding collisions and showcasing robustness as in Fig. \ref{fig:robustness} (a) \& (b). 

\subsubsection{Multiple quadrotors with C3BF-QPs}
In multi-agent scenarios, where multiple quadrotors employ the collision cone CBF-QP (Fig. \ref{fig:robustness} (c)), both the ego-quadrotor and the approaching quadrotor are able to avoid collision in different configurations (static or moving), thus demonstrating robustness with respect to obstacles following the same Collision Cone CBF controller. 

\begin{figure}
       \centering
        \begin{subfigure}[b]{0.46\textwidth}
        \includegraphics[width=\textwidth]{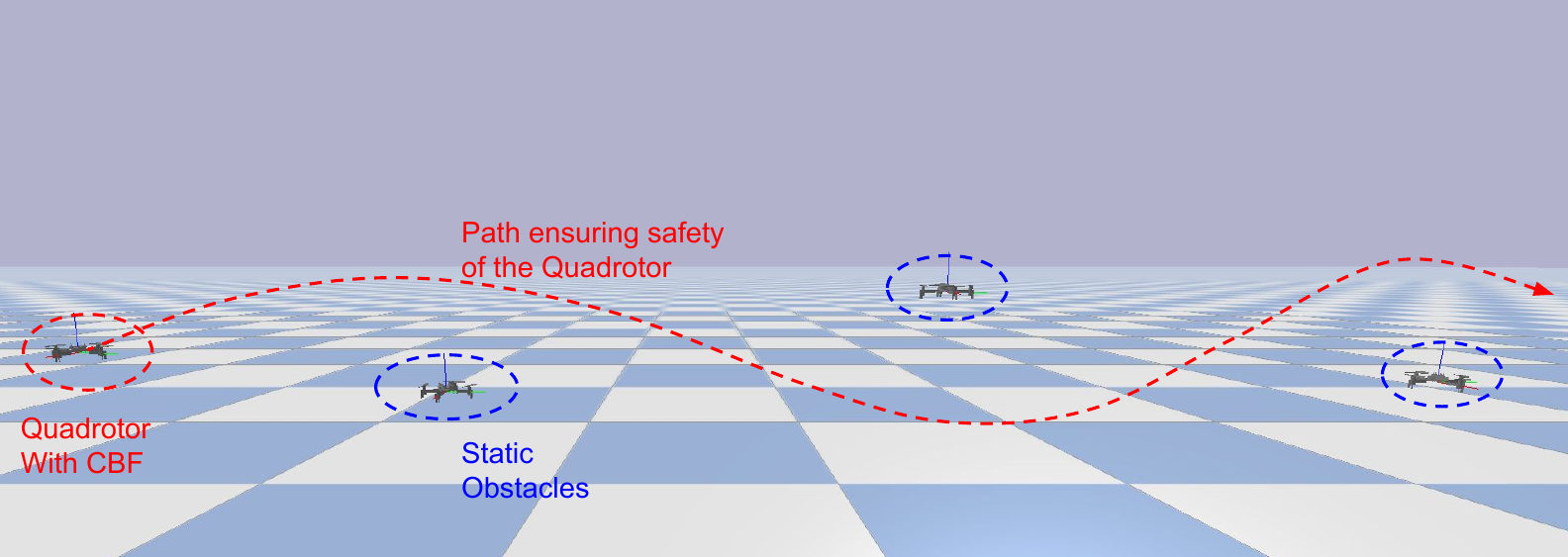}
        \caption{}
        \end{subfigure}
        \begin{subfigure}[b]{0.155\textwidth}
        \includegraphics[width=\textwidth]{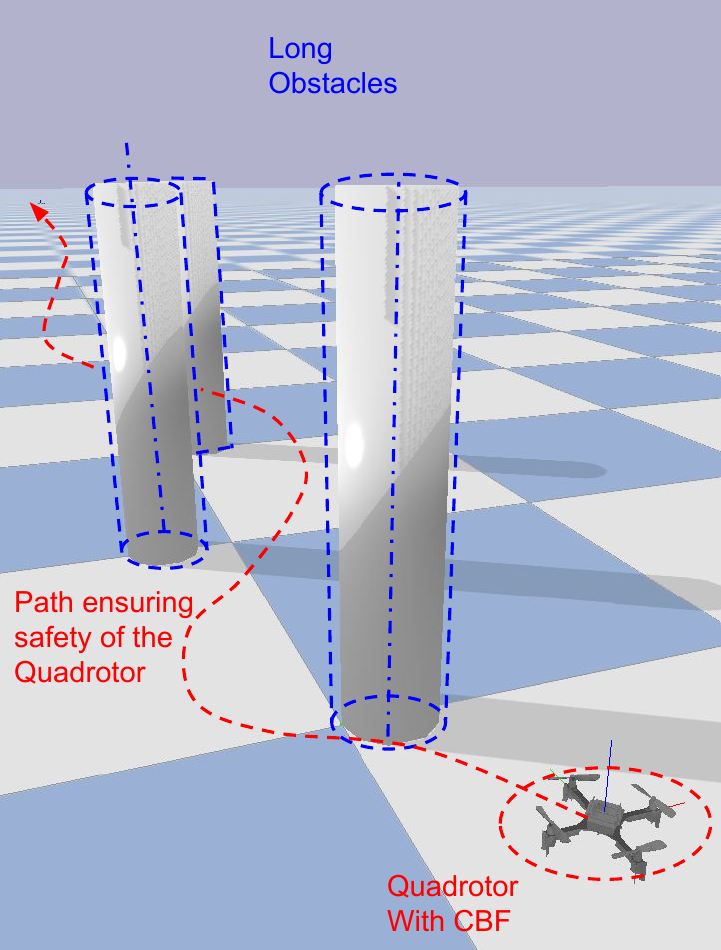}
        \caption{}
        \end{subfigure}
        \begin{subfigure}[b]{0.29\textwidth}
        \includegraphics[width=\textwidth]{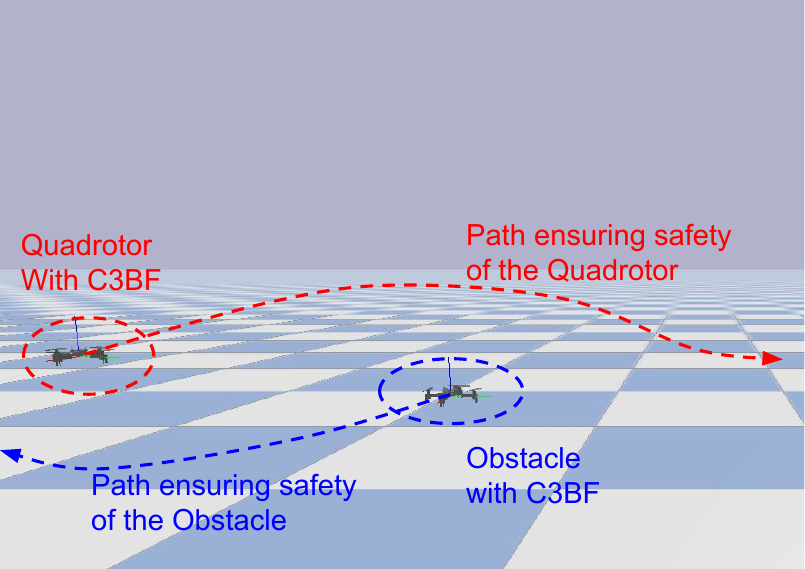}
        \caption{}
        \end{subfigure}
        \caption{Robustness in scenarios with multiple obstacles (a), (b) and with obstacle also following Collision Cone CBF(c).}
        \label{fig:robustness}
    \end{figure}

\section{Conclusions}
\label{section: Conclusions}
We presented the extension of a novel collision cone CBF formulation for quadrotors to avoid collision with static and moving obstacles of various shapes and sizes.
We successfully constructed CBF-QPs with the proposed CBF for the quadrotor model and guaranteed safety by avoiding moving obstacles. This includes collision avoidance with spherical and cylindrical obstacles.
We also showed that the 
current state-of-the-art Higher Order CBFs is more conservative and fails in certain scenarios (shown in the video). 
Finally, we demonstrated the robustness of the proposed CBF-QP controller for safe navigation in a cluttered environment consisting of multiple obstacles and agents with the same safety filters.
In our future work, we plan to implement the controller on quadrotors in real-world situations, interacting with a variety of obstacles. We also intend to explore applications such as the safe teleoperation of quadrotors. Additionally, we aim to investigate the potential of applying the C3BF formulation to legged robots walking in confined spaces. 

\label{section: References}
\bibliographystyle{IEEEtran}
\bibliography{references.bib}

\end{document}